\documentclass[letterpaper]{article}
\usepackage{proceed2e}
\usepackage[margin=1in]{geometry}
\usepackage{times}
\usepackage{float}
\usepackage[caption = false]{subfig}
\usepackage{graphicx}
\usepackage[hidelinks]{hyperref}
\usepackage{enumitem}
\usepackage{xspace}
\usepackage{amsmath}
\usepackage{amsfonts}
\usepackage{amssymb}
\usepackage{amsthm}
\usepackage{bm}
\usepackage{bbm}
\usepackage{mathtools}
\usepackage{enumitem}
\usepackage{thmtools,thm-restate}
\usepackage{algorithm}
\usepackage{algorithmic}
\usepackage{natbib}
\usepackage{color}
\usepackage{graphicx}
\usepackage{comment}
\usepackage[latin1]{inputenc} % for German

\theoremstyle{plain}
\newtheorem{lemma}{Lemma}%[section]
\newtheorem{theorem}{Theorem}%[section]
\newtheorem{proposition}{Proposition}%[section]
\theoremstyle{definition}
\theoremstyle{remark}
\def\DD{\mathcal{D}}

\def\KK{\mathcal{K}}

\def\RR{\mathcal{R}}
\def\SS{\mathcal{S}}

\def\Abb{\mathbb{A}}
\def\Ebb{\mathbb{E}}

\def\Rbb{\mathbb{R}}
\def\Sbb{\mathbb{S}}

\def\R{\Rbb}

\newcommand{\norm}[1]{ \| #1  \|  }

\newcommand{\lr}[2]{\langle #1, #2 \rangle}
\DeclareMathOperator*{\argmin}{arg\,min}

\DeclarePairedDelimiter\floor{\lfloor}{\rfloor}

\newcommand{\E}{\Ebb}

\usepackage{etex,etoolbox}
\usepackage{environ}

\makeatletter
\providecommand{\@fourthoffour}[4]{#4}
\newcommand\fixstatement[2][\proofname\space of]{%
	\ifcsname thmt@original@#2\endcsname
	\AtEndEnvironment{#2}{%
		\xdef\pat@label{\expandafter\expandafter\expandafter
			\@fourthoffour\csname thmt@original@#2\endcsname\space\@currentlabel}%
		\xdef\pat@proofof{\@nameuse{pat@proofof@#2}}%
	}%
	\else
	\AtEndEnvironment{#2}{%
		\xdef\pat@label{\expandafter\expandafter\expandafter
			\@fourthoffour\csname #1\endcsname\space\@currentlabel}%
		\xdef\pat@proofof{\@nameuse{pat@proofof@#2}}%
	}%
	\fi
	\@namedef{pat@proofof@#2}{#1}%
}

\globtoksblk\prooftoks{1000}
\newcounter{proofcount}

\NewEnviron{proofatend}{%
	\edef\next{%
		\noexpand\begin{proof}[\pat@proofof\space\pat@label]%
			\unexpanded\expandafter{\BODY}}%
		\global\toks\numexpr\prooftoks+\value{proofcount}\relax=\expandafter{\next\end{proof}}
	\stepcounter{proofcount}}

\def\printproofs{%
	\count@=\z@
	\loop
	\the\toks\numexpr\prooftoks+\count@\relax
	\ifnum\count@<\value{proofcount}%
	\advance\count@\@ne
	\repeat}
\makeatother

\fixstatement{lemma}
\fixstatement{theorem}
\fixstatement{proposition}
\fixstatement{corollary}

\def\pg{\textsc{policy gradient}\xspace}
\def\ideal{{Ideal}\xspace} 
\def\trpo{\textsc{trpo}\xspace} 
\def\thor{\textsc{thor}\xspace} 
\def\dagger{\textsc{DAgger}\xspace} 
\def\daggered{\textsc{DAggereD}\xspace} 
\def\aggrevate{\textsc{AggreVatTe}\xspace}
\def\aggrevated{\textsc{AggreVaTeD}\xspace}
\def\loki{\textsc{loki}\xspace} 
\def\slols{\textsc{slols}\xspace} 
\def\lols{\textsc{lols}\xspace}
\def\slolsfull{simple \lols\xspace} 
\def\lokifull{Locally Optimal search after $K$-step Imitation\xspace}

\title{Fast Policy Learning through Imitation and Reinforcement}

\def\school{Georgia Tech}
\def\city{Atlanta, GA 30332}
\author{Ching-An Cheng \\
	\school\\
	\city
	\And
	Xinyan Yan \\
	\school\\
	\city
	\And  Nolan Wagener \\
	\school\\
	\city
	\And Byron Boots\\
	\school\\
	\city
}

\begin{document}

\maketitle

\begin{abstract}
	Imitation learning (IL) consists of a set of tools that leverage expert demonstrations to quickly learn policies. 
	However, if the expert  is suboptimal, IL can yield policies with inferior performance compared to reinforcement learning (RL). 
	In this paper, we aim to provide an algorithm that combines the best aspects of RL and IL. We accomplish this by formulating several popular RL and IL algorithms in a common mirror descent framework, showing that these algorithms can be viewed as a variation on a single approach.
	We then propose  
	 \loki, a strategy for policy learning that  
	first performs a small but random number of IL iterations before switching to a policy gradient RL method. We show that if the switching time is properly randomized, \loki can learn to outperform a suboptimal expert and converge faster than running policy gradient from scratch. Finally, we evaluate the performance of \loki experimentally in several  simulated environments.	
\end{abstract}

\section{INTRODUCTION}

Reinforcement learning (RL) has emerged as a promising technique to tackle complex sequential decision problems. When empowered with deep neural networks, RL has demonstrated impressive performance in a range of  synthetic domains~\citep{mnih2013playing,silver2017mastering}. However, one of the major drawbacks of RL is the enormous number of interactions required to learn a policy. This can lead to prohibitive cost and slow convergence when applied to real-world problems, such as those found in robotics~\citep{pan2017agile}.

Imitation learning (IL) has been proposed as an alternate strategy for faster policy learning that works by leveraging additional information provided through expert demonstrations~\citep{pomerleau1989alvinn,schaal1999imitation}.
However, despite significant recent breakthroughs in our understanding  of imitation learning~\citep{ross2011reduction,cheng2018convergence}, the performance of IL is still highly dependent on the quality of the expert policy. When only a suboptimal expert is available, policies learned with standard IL can be inferior to the policies learned by tackling the RL  problem directly with approaches such as policy gradients.

Several recent attempts have endeavored to combine RL and IL~\citep{ross2014reinforcement,chang2015learning,nair2017overcoming,rajeswaran2017learning,sun2018truncated}.
These approaches incorporate the cost information of the RL problem into the imitation process, so the learned policy can \emph{both} improve faster than their RL-counterparts 
and outperform the suboptimal expert policy.
Despite reports of improved  empirical performance, the theoretical understanding of these combined algorithms are still fairly limited~\citep{rajeswaran2017learning,sun2018truncated}. Furthermore, some of these algorithms have requirements that can be difficult to satisfy in practice, such as state resetting~\citep{ross2014reinforcement,chang2015learning}.

In this paper, we aim to provide an algorithm that combines the best aspects of RL and IL. We accomplish this by
first formulating first-order RL and IL algorithms in a common mirror descent framework,  
and show that these algorithms can be viewed as a single approach that only differs in the choice of first-order oracle. On the basis of this new insight, we  address the difficulty of combining IL and RL with a simple, randomized algorithm, named \loki (\lokifull).
As its name suggests, \loki operates in two phases: picking $K$ randomly, it first performs $K$ steps of online IL and then improves the policy with a policy gradient method afterwards. Compared with previous methods that aim to combine RL and IL, \loki is extremely straightforward to implement. Furthermore, it has stronger theoretical guarantees: by properly randomizing $K$, \loki performs as if directly running policy gradient steps with the expert policy as the initial condition. Thus, not only can \loki improve faster than common RL methods, but it can also significantly outperform a suboptimal expert. This is in contrast to previous methods, such as \aggrevate~\citep{ross2014reinforcement}, which generally cannot learn a policy that is better than a one-step improvement over the expert policy.
In addition to these theoretical contributions, we validate the performance of \loki in multiple simulated  environments. The empirical results corroborate our theoretical 
findings.

\section{PROBLEM DEFINITION}
We consider solving discrete-time $\gamma$-discounted infinite-horizon RL problems.\footnote{\loki can be easily  adapted to finite-horizon problems.} Let $\Sbb$ and $\Abb$ be the state and the action spaces, and let $\Pi$ be the policy class.
The objective is to find a policy $\pi\in\Pi$ that minimizes an accumulated cost $J(\pi)$ defined as 
\begin{align} \textstyle
	\min_{\pi \in \Pi} J(\pi), \quad J(\pi) \coloneqq   \E_{\rho_{\pi}} \left[ \sum_{t=0}^{\infty} \gamma^t c(s_t, a_t) \right],  \label{eq:RL problem}
\end{align}
in which $s_t \in \Sbb$, $a_t \in \Abb$, $c$ is the instantaneous cost, 
and $\rho_\pi$ denotes the distribution of trajectories $(s_0, a_0, s_1, \dots)$ generated by running the stationary policy $\pi$ starting from $s_0  \sim p_0(s_0)$.

We denote $Q_{\pi}(s,a)$
as the Q-function under policy $\pi$ and $V_{\pi}(s) = \E_{a \sim \pi_s }[Q_{\pi}(s,a)]$ as the associated value function, where $\pi_s$ denotes the action distribution given state $s$. 
In addition,  we denote $d_{\pi,t}(s)$ as the state distribution at time $t$ generated by running the policy $\pi$ for the first $t$ steps, and we define a joint distribution $d_{\pi}(s,t) = (1-\gamma) d_{\pi,t}(s) \gamma^t$ which has support $\SS \times [0, \infty)$. Note that, while we use the notation $\E_{a \sim \pi} $, the policy class $\Pi$ can be either deterministic or stochastic.

We generally will not deal with the objective function in \eqref{eq:RL problem} directly.  
Instead, we consider a surrogate  problem 
\begin{align}
\min_{\pi \in \Pi} \E_{s,t \sim d_\pi} \E_{a \sim \pi_s } [ A_{\pi'}(s,a) ], \label{eq:normalized obj}
\end{align}
where $A_{\pi'} = Q_{\pi'} - V_{\pi'}$  is the (dis)advantage function with respect to some fixed reference policy $\pi'$. For compactness of writing, we will often omit the random variable in expectation; e.g., the objective function in~\eqref{eq:normalized obj} will be written as $\E_{d_\pi} \E_{\pi} [ A_{\pi'}  ]$ for the remainder of paper.

By the performance difference lemma  below~\citep{kakade2002approximately}, it is easy to see that  solving~\eqref{eq:normalized obj} is equivalent to solving~\eqref{eq:RL problem}.
\begin{lemma} \label{lm:performance difference}
	\citep{kakade2002approximately} Let $\pi$ and $\pi'$ be two policies and  $ A_{\pi'}(s, a) = Q_{\pi'}(s,a) - V_{\pi'}(s)$ be the (dis)advantage function  with respect to running $\pi'$. Then it holds that \vspace{-2mm}
	\begin{align} \label{eq:performance difference}
	J(\pi) = J(\pi') + \frac{1}{1-\gamma}\E_{d_\pi} \E_{\pi} [ A_{\pi'}  ].   
	\end{align}
\end{lemma}

\vspace{-3mm}
\section{FIRST-ORDER RL AND IL} 

We formulate both first-order RL and IL methods within a single mirror descent framework~\citep{nemirovski2009robust}, which includes common update rules~\citep{sutton2000policy,kakade2002natural,peters2008natural,peters2010relative,  rawlik2012stochastic,silver2014deterministic,schulman2015high,ross2011reduction,sun2017deeply}. 
We show that policy updates based on RL and IL mainly differ in first-order stochastic oracles used, as summarized in Table~\ref{table:oracle_comparison}.
\begin{table*}[h]
	\caption[Comparison]{Comparison of First-Order Oracles
    }
	\label{table:oracle_comparison}
	\begin{center}
		\begin{tabular}{ll}
			Method & First-Order Oracle \\\hline\\
			\pg (Section~\ref{sec:policy gradient description}) &  $\E_{d_{\pi_n}}\left( \nabla_{\theta}  \E_{ \pi} \right)  [ A_{\pi_n}  ]$ \\
			\daggered (Section~\ref{sec:imitation gradient description})& $\E_{d_{\pi_n}}\left( \nabla_{\theta}  \E_{ \pi} \right) [\E_{\pi^*} [ d ]]$
             \\
			\aggrevated (Section~\ref{sec:imitation gradient description})& $\E_{d_{\pi_n}}\left( \nabla_{\theta}  \E_{ \pi} \right)  [ A_{\pi^*}  ]$ \\
			\slols (Section~\ref{sec:related work}) &  $ \E_{d_{\pi_n}} \left( \nabla_{\theta}  \E_{ \pi} \right) [ (1-\lambda) A_{\pi_n} + \lambda A_{\pi^*} ]$ \\
			\thor (Section~\ref{sec:related work}) &$ \E_{d_{\pi_n}} \left( \nabla_{\theta}  \E_{ \pi} \right) [ A^{H,\pi^*}_{\pi_n,t} ]$ %\footnotemark \\ 	            
		\end{tabular}        
	\end{center}
    \vspace{-2mm}
\end{table*}

\subsection{MIRROR DESCENT}

We begin by defining the iterative rule to update policies.
We assume that the learner's policy $\pi$ is parametrized by some $\theta \in \Theta $, where $\Theta$ is a closed and convex set, and that the learner has access to a family of strictly convex functions $\RR$. 

To update the policy, in the $n$th iteration, the learner receives a vector $g_n$ from a first-order oracle, picks $R_n \in \RR$, and then performs a mirror descent step:
\begin{align} \label{eq:mirror descent}
\theta_{n+1} &= P_{n,g_n}(\theta_n)
\end{align}
where $P_{n,g_n}$ is a prox-map defined as
\begin{align} \label{eq:prox-map}
P_{n,g_n}(\theta_n)
&\coloneqq \argmin_{\theta \in \Theta} \lr{g_n}{\theta} + \frac{1}{\eta_n}D_{R_n}(\theta||\theta_n).
\end{align}
 $\eta_n$ is the step size,
and  $D_{R_n}$ is the Bregman divergence associated with $R_n$~\citep{bregman1967relaxation}:
$
D_{R_n}(\theta||\theta_n) \coloneqq R_n(\theta) - R_n(\theta_n) - \lr{\nabla R_n(\theta_n)}{\theta - \theta_n}
$.

By choosing proper $R_n$, the mirror descent framework in~\eqref{eq:mirror descent} covers most RL and IL algorithms. 
Common choices of $R_n$ include negative entropy~\citep{ peters2010relative, rawlik2012stochastic}, $\frac{1}{2}\norm{\theta}_2^2$~\citep{sutton2000policy,silver2014deterministic}, and $\frac{1}{2}\theta^\top F(\theta_n) \theta$ with $F(\theta_n)$ as the Fisher information matrix~\citep{kakade2002natural, peters2008natural, schulman2015trust}.

\subsection{FIRST-ORDER ORACLES}

While both first-order RL and IL methods can be viewed as performing mirror descent, they differ in the choice of the first-order oracle that returns the update direction $g_n$.
Here we show the vector $g_n$ of both approaches can be derived as a stochastic approximation of the (partial) derivative of $\E_{d_\pi} \E_{\pi} [ A_{\pi'}  ]$ with respect to policy $\pi$, but with a different reference policy $\pi'$.

\subsubsection{Policy Gradients} \label{sec:policy gradient description}
A standard approach to RL is to treat~\eqref{eq:RL problem} as a stochastic nonconvex optimization problem. In this case, $g_n$ in mirror descent~\eqref{eq:mirror descent} is an estimate of the policy gradient $\nabla_{\theta} J(\pi)$~\citep{williams1992simple,sutton2000policy}.

To compute the policy gradient in the $n$th iteration, we set the current policy $\pi_n$ as the reference policy in~\eqref{eq:performance difference} (i.e. $\pi' = \pi_n$),  which is treated as constant in $\theta$ in the following policy gradient computation. 
Because 
$
\E_{\pi_n} [ A_{\pi_n}]  = \E_{\pi_n} [ Q_{\pi_n} ]  - V_{\pi_n}  = 0  
$, using \eqref{eq:performance difference}, the policy gradient can be written as\footnote{We assume the cost is sufficiently regular so that the order of differentiation and expectation can exchange.}
\begin{align}   \label{eq:policy gradient}
&(1-\gamma) \nabla_{\theta} J(\pi) |_{\pi=\pi_n} \nonumber \\
&= \nabla_\theta \E_{d_\pi} \E_{\pi} [ A_{\pi_n}] \vert_{\pi = \pi_n}  \nonumber  \\
&=  \left( \nabla_{\theta} \E_{ d_\pi} \right) [0 ]   +   \E_{d_{\pi}}  \left( \nabla_{\theta}  \E_{ \pi} \right)  [ A_{\pi_n}  ] \vert_{\pi = \pi_n}  \nonumber \\
&=  \E_{ d_\pi}  \left( \nabla_{\theta}  \E_{ \pi} \right)  [  A_{\pi_n} ]  \vert_{\pi = \pi_n}
\end{align}
The above expression is unique up to a change of baselines: $\left( \nabla_{\theta}  \E_{ \pi} \right)  [ A_{\pi_n}]$ is equivalent to $ \left( \nabla_{\theta}  \E_{ \pi} \right)  [ A_{\pi_n} + b]$, because $ \left(\nabla_{\theta}  \E_{ \pi} \right) [ b(s) ]  = \nabla_{\theta} b(s)  = 0$, where  $b:\Sbb \to \R$ is also called a control variate~\citep{greensmith2004variance}.

The exact formulation of  $ \left( \nabla_{\theta}  \E_{\pi} \right)  [ A_{\pi_n}  ]$ depends on whether the policy $\pi$ is stochastic or deterministic. For stochastic policies,\footnote{A similar equation holds for reparametrization~\citep{grathwohl2017backpropagation}.}  we can compute it with the likelihood-ratio method  and write
\begin{align} \label{eq:stochastic policy derivative}
\left( \nabla_{\theta}  \E_{ \pi} \right)  [ A_{\pi_n}  ] =  \E_{ \pi}  [ A_{\pi_n}  \nabla_{\theta}  \log  \pi   ]  
\end{align}
For deterministic policies, we replace the expectation as evaluation (as it is the expectation over a Dirac delta function, i.e. $a = \pi(s)$) and use the chain rule:
\begin{align} \label{eq:deterministic policy derivative}
\left( \nabla_{\theta}  \E_{\pi} \right)  [ A_{\pi_n}  ] = \nabla_{\theta} A_{\pi_n} (s,\pi)  
=  \nabla_{\theta} \pi \nabla_{a} A_{\pi_n}  
\end{align}
Substituting~\eqref{eq:stochastic policy derivative} or~\eqref{eq:deterministic policy derivative} back into~\eqref{eq:policy gradient}, we get the equation for stochastic policy gradient~\citep{sutton2000policy} or deterministic policy gradient~\citep{silver2014deterministic}. 
Note that the above equations require the exact knowledge, or an unbiased estimate, of $A_{\pi}$. In practice, these terms are further approximated using function approximators, leading to biased gradient estimators~\citep{konda2000actor,schulman2015high,mnih2016asynchronous}.

\subsubsection{Imitation Gradients} \label{sec:imitation gradient description}

An alternate strategy to RL is IL. In particular, we consider \emph{online} IL, which interleaves data collection and policy updates to overcome the covariate shift problem of traditional batch IL~\citep{ross2011reduction}.
Online IL assumes that a (possibly suboptimal) expert policy $\pi^*$ is available as a black-box oracle, from which demonstrations $a^* \sim \pi^*(s)$  can be queried for any given state $s \in \Sbb$. Due to this requirement, the expert policy in online IL is often an \textit{algorithm} (rather than a human demonstrator), which is hard-coded or based on additional computational resources, such as trajectory optimization~\citep{pan2017agile}. 
The goal of IL is to learn a policy that can perform similar to, or better than, the expert policy.

Rather than solving the stochastic nonconvex optimization directly, online IL solves an online learning problem with per-round cost in the $n$th iteration defined as
\begin{align} \label{eq:per-round cost}
l_n(\pi) =\E_{d_{\pi_n}} \E_{ \pi} [ \tilde{c}  ]
\end{align}
where $\tilde{c}:\Sbb\times\Abb\to\R$ is a surrogate loss satisfying the following condition: For all $s \in \Sbb$ and $\pi \in \Pi$, there exists a constant $C_{\pi^*} > 0$ such that 
\begin{align} \label{eq:requirement of imitation}
C_{\pi^*}  \E_{ \pi} [ \tilde{c}  ] \geq  \E_{ \pi} [ A_{\pi^*}  ].
\end{align}
 By Lemma~\ref{lm:performance difference}, this implies
$ J(\pi_n) \leq J(\pi^*) +  \frac{C_{\pi^*} }{1-\gamma}l_n(\pi_n)$. Namely, in the $n$th iteration, online IL attempts to minimize an online upper-bound of $J(\pi_n)$.

\dagger~\citep{ross2011reduction} chooses 
$\tilde{c}$ to be a strongly convex function $\tilde{c}(s,a) = \E_{a^*\sim \pi^*(s)}[d(a,a^*)]$ that penalizes the difference between the learner's policy and the expert's policy, where $d$ is some metric of space $\Abb$ (e.g., for a continuous action space \citet{pan2017agile} choose $d(a,a^*) = \norm{a -a^*}^2$).
More directly, \aggrevate simply chooses $\tilde{c}(s,a) = A_{\pi^*}(s,a)$~\citep{ross2014reinforcement}; in this case, the policy learned with online IL can potentially outperform the expert policy.

First-order online IL methods operate by updating policies with mirror descent~\eqref{eq:mirror descent} with $g_n$ as an estimate of 
 \begin{align} \label{eq:imitation gradient}
\nabla_{\theta} l_n(\pi_n) = \E_{d_{\pi_n}} \left( \nabla_\theta \E_{ \pi} \right) [ \tilde{c}  ] |_{\pi = \pi_n}
 \end{align}
Similar to policy gradients, the implementation of~\eqref{eq:imitation gradient}  
can be executed using either~\eqref{eq:stochastic policy derivative} or~\eqref{eq:deterministic policy derivative} (and with a control variate).
One particular case of~\eqref{eq:imitation gradient}, with  $\tilde{c} = A_{\pi^*}$, is known as \aggrevated~\citep{sun2017deeply},
 \begin{align} \label{eq:aggrevated gradient}
 	\nabla_\theta l_n(\pi_n) =\E_{d_{\pi_n}} \left( \nabla_\theta \E_{ \pi} \right) [A_{\pi^*} ] |_{\pi = \pi_n}.
 \end{align}
Similarly, we can turn \dagger into a first-order method, which we call \daggered, by using $g_n$ as an estimate of the  first-order oracle 
\begin{align} \label{eq:daggered gradient}
\nabla_\theta l_n(\pi_n) =\E_{d_{\pi_n}} \left( \nabla_\theta \E_{ \pi} \right)  \E_{\pi^*(s)}[d].
\end{align}
A comparison is summarized in Table~\ref{table:oracle_comparison}.

\section{THEORETICAL COMPARISON} \label{sec:comparison}

With the first-order oracles defined, we now compare the performance and properties of performing mirror descent with policy gradient or imitation gradient. We will see that while both approaches share the same update rule in~\eqref{eq:mirror descent}, the generated policies have different behaviors: using policy gradient generates a \emph{monotonically} improving policy sequence, whereas using imitation gradient generates a policy sequence that improves \emph{on average}. 
Although the techniques used in this section are not completely new in the optimization literature, we specialize the results to compare performance and to motivate \loki in the next section.
 The proofs of this section are included in Appendix~\ref{app:proof for comparison}.

\subsection{POLICY GRADIENTS}
We analyze the performance of policy gradients with standard techniques from nonconvex analysis.
\begin{proposition} \label{pr:nonconvex performance}
	Let $J$ be $\beta$-smooth and $R_n$ be $\alpha_n$-strongly convex  with respect to norm $\norm{\cdot}$. 
	Assume $\E[g_n] = \nabla_{\theta} J(\pi_n)$. 	
	For $ \eta_n \leq \frac{2\alpha_n}{\beta} $, it satisfies
	\begin{align*} \textstyle
	\E\left[ J(\pi_{n+1}) \right] &\textstyle
	\leq J(\pi_0 ) +    \E\left[ \sum_{n=1}^{N}\frac{2\eta_n}{\alpha_n}   \norm{\nabla_\theta J(\pi_n) - g_n}_*^2 \right]\\
	&\textstyle  +  \frac{1}{2}\E\left[ \sum_{n=1}^{N}\left( - \alpha_n \eta_n +  \frac{ \beta \eta_n^2}{2} \right)   \norm{\hat{\nabla}_\theta J(\pi_n) }^2 \right] 
	\end{align*}
	where the expectation is due to randomness of sampling $g_n$, and 
     $ \hat{\nabla}_\theta J(\pi_n) \coloneqq \frac{1}{\eta_n} \left( \theta_{n} - P_{n,\nabla_\theta J(\pi_n)}(\theta_n) \right)$.
	 is a gradient surrogate.
\end{proposition}
Proposition~\ref{pr:nonconvex performance} shows that monotonic improvement can be made under proper smoothness assumptions if the step size is small and noise is comparably small with the gradient size. 
However, 
the final policy's performance is sensitive to the initial condition $J(\pi_0)$, which can be poor for a randomly initialized policy.

Proposition~\ref{pr:nonconvex performance} also suggests that the size of the gradient $\norm{\hat{\nabla}_\theta J(\pi_n) }^2$  does not converge to zero on average. 
Instead, it converges to a size proportional to the sampling noise of policy gradient estimates due to the linear dependency of $\frac{2\eta_n}{\alpha_n}   \norm{\nabla_\theta J(\pi_n) - g_n}_*^2$ on $\eta_n$. This phenomenon is also mentioned by~\cite{ghadimi2016mini}. We note that this pessimistic result is because the prox-map~\eqref{eq:prox-map} is nonlinear in $g_n$ for general $R_n$ and $\Theta$. However,  when $R_n$ is quadratic and $\Theta$ is unconstrained, the convergence of $\norm{\hat{\nabla}_\theta J(\pi_n) }^2$ to zero on average can be guaranteed (see Appendix~\ref{app:proof of th:nonconvex performance} for a discussion).

\subsection{IMITATION GRADIENTS}

While applying mirror descent with a policy gradient can generate a monotonically improving policy sequence, applying the same algorithm with an imitation gradient yields a different behavior. 
The result is summarized below, which is a restatement of~\cite[Theorem 2.1]{ross2014reinforcement}, but is specialized for mirror descent.
\begin{proposition} \label{pr:online convex performance}
Assume $l_n$ is $\sigma$-strongly convex with respect to $R_n$.\footnote{A function  $f$ is said to be $\sigma$-strongly convex with respect to $R$ on a set $\KK$ if 
for all $x, y\in \KK$, $f(x) \geq f(y) + \lr{\nabla f(y)}{x-y} + \sigma D_R(x||y)$. }	Assume $\E[g_n] = \nabla_{\theta} l_n(\pi_n)$ and $\norm{g_n}_*\leq G< \infty$ almost surely. For $\eta_n = \frac{1}{\hat{\sigma} n}$ with $\hat{\sigma} \leq \sigma$, it holds
\begin{align*} \textstyle
\frac{1}{N}\E\left[ \sum_{n=1}^{N} J(\pi_n) \right] \leq J(\pi^*) +  \frac{C_{\pi^*}}{1-\gamma} \left( \epsilon_{\text{class}} + \epsilon_{\text{regret}} \right) 
\end{align*}
where the expectation is due to randomness of sampling $g_n$,
$
\epsilon_{\text{class}} = \sup_{\{\pi_n\}} \inf_{\pi\in\Pi}  \frac{1}{N}\sum_{n=1}^{N} l_n(\pi)
$
 and
$
\epsilon_{\text{regret}}= \frac{G^2 (\log N +1) }{2\hat{\sigma} N }$.
\end{proposition}

Proposition~\ref{pr:online convex performance} is based on the assumption that $l_n$ is strongly convex, which can be verified for certain problems~\citep{cheng2018convergence}.
Consequently, Proposition~\ref{pr:online convex performance} shows that the performance of the policy sequence on average can converge close to the expert's performance $J(\pi^*)$, with additional error that is proportional to $\epsilon_{\text{class}}$ and $\epsilon_{\text{regret}}$. 

$\epsilon_{\text{regret}}$ is an upper bound of the average regret, which is less than $\tilde{O}(\frac{1}{N})$ for a \emph{large} enough step size.\footnote{The step size should be large enough to guarantee $\tilde{O}(\frac{1}{N})$ convergence, where $\tilde{O}$ denotes Big-O but omitting $\log$ dependency. However, it should be bounded since $\epsilon_{\text{regret}} = \Theta\left( \frac{1}{\hat{\sigma}} \right)$.} This characteristic is in contrast to policy gradient, which requires \emph{small} enough step sizes to guarantee local improvement. 
 
$\epsilon_{\text{class}}$ measures the expressiveness of the policy class $\Pi$. It can be \textit{negative} if there is a policy in $\Pi$ that outperforms the expert policy $\pi^*$ in terms of $\tilde{c}$. 
However, 
since online IL attempts to minimize an online upper bound of the accumulated cost through a surrogate loss $\tilde{c}$, 
the policy learned with imitation gradients in general cannot be better than performing one-step policy improvement from the expert policy~\citep{ross2014reinforcement,cheng2018convergence}. Therefore, when the expert is suboptimal, the reduction from nonconvex optimization to online convex optimization can lead to suboptimal policies.

Finally, we note that updating policies with imitation gradients does not necessarily generate a monotonically improving policy sequence, even for deterministic problems;  whether the policy improves monotonically is completely problem dependent~\citep{cheng2018convergence}.  Without going into details, we can see this by comparing policy gradient in~\eqref{eq:policy gradient} and the special case of imitation gradient in~\eqref{eq:aggrevated gradient}. By  Lemma~\ref{eq:performance difference}, we see that 
\begin{align*}
&\E_{d_{\pi_n}} \left( \nabla_\theta \E_{ \pi} \right) [A_{\pi_n}  ]
\\
& = \left(  \nabla_\theta  \E_{d_{\pi}} \right)\E_{ \pi_n}  [A_{\pi^*}  ] + 
\E_{d_{\pi_n}} \left( \nabla_\theta \E_{ \pi} \right) [A_{\pi^*}  ].
\end{align*}
Therefore, even with $\tilde{c} = A_{\pi^*}$, the negative of the direction in~\eqref{eq:aggrevated gradient} is not necessarily a descent direction; namely applying~\eqref{eq:aggrevated gradient} to update the policy is not guaranteed to improve the policy performance locally.

\section{IMITATE-THEN-REINFORCE} \label{sec:method}

\begin{algorithm}[t]
	{\footnotesize
		\caption{\loki }\label{alg:algo} 
		\begin{algorithmic} [1]
			\renewcommand{\algorithmicrequire}{\textbf{Parameters:}}
			\renewcommand{\algorithmicensure}{\textbf{Input:}}		
			\REQUIRE $d$, $N_m$, $N_M$
			\ENSURE $\pi^*$ 
			\STATE Sample  $K$ with probability in~\eqref{eq:K_probability}.			
			\FOR[$\quad$\# Imitation Phase]{$t = 1\dots K$}
				\STATE Collect data $\DD_n$ by executing $\pi_n$ 				
				\STATE Query $g_n$ from~\eqref{eq:imitation gradient} using $\pi^*$
				\STATE Update $\pi_n$ by mirror descent ~\eqref{eq:prox-map} with $g_n$
				\STATE Update advantage function estimate $\hat A_{\pi_n}$ by $\DD_n$
			\ENDFOR
			\FOR[$\quad$\# Reinforcement Phase]{$t = K+1 \dots$}
				\STATE Collect data $\DD_n$ by executing $\pi_n$.
				\STATE Query $g_n$ from~\eqref{eq:policy gradient} f using $\hat A_{\pi_n}$
				\STATE Update $\pi_n$ by mirror descent ~\eqref{eq:prox-map} with $g_n$
				\STATE Update advantage function estimate $\hat A_{\pi_n}$ by $\DD_n$
			\ENDFOR
		\end{algorithmic}
	} 
\end{algorithm}

To combine the benefits from RL and IL, we propose a simple randomized algorithm \loki: first perform $K$ steps of mirror descent with imitation gradient and then switch to policy gradient for the rest of the steps. Despite the algorithm's simplicity, we show that, when $K$ is appropriately randomized, running \loki has similar performance to performing policy gradient steps directly from the expert policy. 

\subsection{ALGORITHM: \loki}

The algorithm \loki is summarized in Algorithm~\ref{alg:algo}. The algorithm is composed of two phases: an imitation phase and a reinforcement phase. In addition to learning rates, \loki receives three hyperparameters ($d$, $N_m$, $N_M$) which determine the probability of random switching at time $K$. As shown in the next section, these three hyperparameters can be selected fairly simply.

\paragraph{Imitation Phase}
Before learning, \loki first randomly samples a number $K \in [N_m, N_M]$ according to the prescribed probability distribution~\eqref{eq:K_probability}. Then it performs $K$ steps of mirror descent with imitation gradient. 
In our implementation, we set 
\begin{align} \label{eq:KL divergence}
\E_{\pi}[\tilde{c}] = \text{KL}(\pi^*||\pi),
\end{align}
which is the KL-divergence between the two policies. It can be easily shown that a proper constant $C^*$ exists satisfying the requirement of $\tilde{c}$ in~\eqref{eq:requirement of imitation}~\citep{gibbs2002choosing}. 
While using~\eqref{eq:KL divergence} does not guarantee learning a policy that outperforms the expert due to $\epsilon_{\text{class}} \geq 0$, with another reinforcement phase available, the imitation phase of \loki is only designed to quickly bring the initial policy closer to the expert policy.	
Compared with choosing $\tilde{c} = A_{\pi^*}$ as in  \aggrevated, one benefit of choosing $\text{KL}(\pi^*||\pi)$ (or its variants, e.g. $\norm{a-a^*}^2$) is that it does not require learning a value function estimator. In addition, the imitation gradient can be calculated through \emph{reparametrization} instead of a likelihood-ratio~\citep{tucker2017rebar}, as now $\tilde{c}$ is presented as a differentiable function in $a$. Consequently, the sampling variance of imitation gradient can be significantly reduced by using multiple samples of $a \sim \pi_n$ (with a single query from the expert policy) and then performing averaging.

\paragraph{Reinforcement Phase}
After the imitation phase, \loki switches to the reinforcement phase.
 At this point, the policy $\pi_K$ is much closer to the expert policy than the initial policy $\pi_0$. In addition, an estimate of $A_{\pi_K}$ is also available. Because the learner's policies were applied to collect data in the previous \emph{online} imitation phase, $A_{\pi_n}$ can already be updated accordingly, for example, by minimizing TD error. 
 Compared with other warm-start techniques, \loki can learn \emph{both} the policy and the advantage estimator in the imitation phase.

\subsection{ANALYSIS}

We now present the theoretical properties of \loki. The analysis is composed of two steps. First, we show the performance of $J(\pi_K)$ in  Theorem~\ref{th:weighted random stop}, a generalization of Proposition~\ref{pr:online convex performance} to consider the effects of non-uniform random sampling. Next, combining Theorem~\ref{th:weighted random stop} and Proposition~\ref{pr:nonconvex performance}, we show the performance of \loki in Theorem~\ref{th:switch}. The proofs are given in Appendix~\ref{app:proof of method}.

\begin{theorem} \label{th:weighted random stop}
	Let $d \geq 0 $, $N_m \geq 1$, and $N_M \geq 2  N_m $. Let $K \in [N_m, N_M]$ be a discrete random variable such that 
	\begin{align} \label{eq:K_probability} 
	\textstyle
	P(K = n ) = n^d / \sum_{m=N_m}^{N_M} m^d.
	\end{align}
	Suppose $l_n$ is $\sigma$-strongly convex with respect to $R_n$,  $\E[g_n] = \nabla_{\theta} l_n(\pi_n)$, and $\norm{g_n}^* \leq G < \infty$ almost surely.
	Let $\{\pi_n\}$ be generated by running mirror descent with step size
	$\eta_n = n^d/  \hat{\sigma} \sum_{m=1}^{n} m^d$. 
	For $\hat{\sigma} \leq \sigma $, it holds that 
	\begin{align*}
	\E\left[ J(\pi_{K}) \right] 
	\leq   J(\pi^*) + \Delta, 
	\end{align*}
	where  the expectation is due to  sampling $K$ 
	and $g_n$, $\Delta = \frac{C_{\pi^*}}{1-\gamma } \left( \epsilon^w_{\text{class}} +
2^{-d}	\hat{\sigma} D_{\RR} 
+   G^2  C_{N_M}/\hat{\sigma} N_M    \right)$,  $D_{\RR} = \sup_{R \in \RR} \sup_{\pi, \pi' \in \Pi} D_R(\pi' || \pi)$, 
	$
	\epsilon^w_{\text{class}} \coloneqq \sup_{\{w_n\},\{\pi_n\}} \inf_{\pi\in\Pi} \frac{\sum_{n=1}^{N} w_n l_n(\pi)}{\sum_{n=1}^{N} w_n}
	$, and 
	\begin{align*}
	C_{N_M} = \begin{cases}
	\log(N_M)+1,  & \text{if $d=0$}\\
	\frac{8d }{3} \exp\left( \frac{d}{N_M} \right), & \text{if $d\geq 1$}
	\end{cases}
	\end{align*}
\end{theorem}

Suppose $N_M \gg d$. Theorem~\ref{th:weighted random stop} says that the performance of $J(\pi_{K})$ in expectation converges to $J(\pi^*)$ in a rate of $\tilde{O}(d/N_M)$ when a proper step size is selected. 
In addition to the convergence rate, we notice that the performance gap between $J(\pi^*)$ and $J(\pi_K)$
is bounded by $O(\epsilon_{\text{class}}^w + 2^{-d} D_\RR)$. $\epsilon_{\text{class}}^w$ is a weighted version of the expressiveness measure of policy class $\Pi$ in Proposition~\ref{pr:online convex performance}, which can be made small if $\Pi$ is rich enough with respect to the \emph{suboptimal} expert policy. $D_\RR$ measures the size of the decision space with respect to the class of regularization functions $\RR$ that the learner uses in mirror descent. The dependency on $D_{\RR}$ is because Theorem~\ref{th:weighted random stop} performs a suffix random sampling with $N_m > 0$.
While the presence of $D_\RR$ increases the gap, its influence can easily made small with a slightly large $d$ due to the factor $2^{-d}$. 

In summary, due to the sublinear convergence rate of IL, $N_M$ does not need to be large (say less than 100) as long as $N_M \gg d$; on the other hand, due to the $2^d$ factor, $d$ is also small (say less than $5$) as long as it is  large enough to cancel out the effects of $D_\RR$. Finally, we note that, like Proposition~\ref{pr:online convex performance}, Theorem~\ref{th:weighted random stop} encourages using larger step sizes, which can further boost the convergence of the policy in the imitation phase of \loki.

Given Proposition~\ref{pr:nonconvex performance} and Theorem~\ref{th:weighted random stop}, now it is fairly easy to understand the performance of \loki.
\begin{theorem} \label{th:switch}
Running \loki holds that
\begin{align*} 
&\textstyle
\E\left[ J(\pi_{N}) \right] \leq J(\pi^* ) + \Delta  \\
&\textstyle
+ \E\left[ \sum_{n=K+1}^{N}\frac{2\eta_n}{\alpha_n}   \norm{\nabla_\theta J(\pi_n) - g_n}_*^2 \right]\\
&\textstyle
  +  \frac{1}{2}\E\left[ \sum_{n=K+1}^{N}\left( - \alpha_n \eta_n +  \frac{ \beta \eta_n^2}{2} \right)   \norm{\hat{\nabla}_\theta J(\pi_n) }^2 \right] ,
\end{align*}
where the expectation is due to  sampling $g_n$ and $K$.
\end{theorem}

Firstly, Theorem~\ref{th:switch} shows that $\pi_N$ can perform better than the expect policy $\pi^*$, and, in fact, it converges to a locally optimal policy on average under the same assumption as in Proposition~\ref{pr:nonconvex performance}.
Compare with to running policy gradient steps directly from the expert policy, running \loki introduces an additional gap $O(\Delta + K \norm{\hat{\nabla}_\theta J(\pi)}^2)$. However, as discussed previously, $\Delta$ and $K \leq N_M \ll N$ are reasonably small, for usual $N$ in RL. Therefore, performing \loki almost has the same effect as using the expert policy as the initial condition, which is the best we can hope for when having access to an expert policy.

We can also compare \loki with performing usual policy gradient updates from a randomly initialized policy. The performance difference can be easily shown as $O( J(\pi^*) - J(\pi_0) + \Delta + K \norm{\hat{\nabla}_\theta J(\pi)}^2)$. Therefore, if  performing $K$ steps of policy gradient from $\pi_0$ gives a policy with performance worse than $J(\pi^*) + \Delta$, then \loki is favorable.

\section{RELATED WORK} \label{sec:related work}

We compare \loki with some recent attempts to incorporate the loss information $c$ of RL into IL so that it can learn a policy that outperforms the expert policy. As discussed in Section~\ref{sec:comparison}, when $\tilde{c} = A_{\pi^*}$, \textsc{AggreVaTe(D)} can potentially learn a policy that is better than the expert policy~\citep{ross2014reinforcement,sun2017deeply}. However, implementing \textsc{AggreVaTe(D)} exactly as suggested by theory can be difficult and inefficient in practice.
On the one hand, while $A_{\pi^*}$ can be learned off-policy using samples collected by running the expert policy, usually the estimator quality  is unsatisfactory due to covariate shift. On the other hand, if $A_{\pi^*}$ is learned on-policy, it requires restarting the system from any state, or requires performing $\frac{1}{1-\gamma}$-times more iterations to achieve the same convergence rate as other choices of $\tilde{c}$ such as  $\text{KL}(\pi^*||\pi) $ in \loki;  both of which are impractical for usual RL problems. 

Recently, \citet{sun2018truncated} proposed \thor (Truncated HORizon policy search) which solves a truncated RL problem with the expert's value function as the terminal loss to alleviate the strong dependency of \aggrevated on the quality of $ A_{\pi^*}$. 
Their algorithm  uses an $H$-step truncated advantage function defined as
{$A^{H,\pi^*}_{\pi_n,t} = \E_{\rho_{\pi_n}} [\sum_{\tau=t}^{t+H-1} \gamma^{\tau -t }c(s_\tau, a_\tau) + \gamma^H V_{\pi^*}(s_{t+H}) - V_{\pi^*}(s_t)]$}.
While empirically the authors show that the learned policy can  improve over the expert policy, the theoretical properties of \thor remain somewhat unclear.\footnote{The algorithm actually implemented by~\citet{sun2018truncated} does not solve precisely the same problem analyzed in theory.} 
In addition, \thor is more convoluted to implement and relies on multiple advantage function estimators. By contrast, \loki has stronger theoretical guarantees, while being straightforward to implement with off-the-shelf learning algorithms.

Finally, we compare \loki with \lols (Locally Optimal Learning
to Search), proposed by~\cite{chang2015learning}. \lols is an online IL algorithm which sets $\tilde{c} = Q_{\hat{\pi}^{\lambda}_n}$, where $\lambda \in [0,1]$ and $\hat{\pi}^{\lambda}_n$ is a mixed policy that at each time step chooses to run the current policy $\pi_{n}$ with probability $1-\lambda$ and the expert policy $\pi^*$ with probability $\lambda$. Like \aggrevated, \lols suffers from the impractical requirement of estimating $Q_{\hat{\pi}^{\lambda}_n}$, which relies on the state resetting assumption. 

Here we show that such difficulty can be addressed by using the mirror descent framework  with $g_n$ as an estimate of $\nabla_{\theta} l_n^{\lambda} (\pi_n) $, where 
$
l_n^{\lambda}(\pi) \coloneqq \E_{d_{\pi_n}} \E_{ \pi}[ (1-\lambda) A_{\pi_n} + \lambda A_{\pi^*} ]
$.
That is, the first-order oracle is simply a convex combination of policy gradient and \aggrevated gradient. 
We call such linear combination~\slols (\slolsfull) and we show it has the same performance guarantee as \lols.
\begin{theorem} \label{th:mix policy}	
Under the same assumption in Proposition~\ref{pr:online convex performance}, 
running \slols generates a policy sequence, with randomness due to sampling $g_n$, satisfying
	\begin{small}
		\begin{align*}
		 \frac{1}{N} \E\left[ \sum_{n=1}^{N}  J(\pi_n)  - \left(  (1-\lambda)J^*_{\pi_n}  + \lambda J(\pi^*) \right)  \right] 
		\leq  \frac{\epsilon^{\lambda}_{\text{class}} + \epsilon^\lambda_{\text{regret}}}{1-\gamma}
		\end{align*}
	\end{small}
	where $J_{\pi_n}^* =  \min_{\pi \in \Pi} \E_{d_{\pi_n}}\E_{\pi}[Q_{\pi_n}] \eqqcolon  \E_{d_{\pi_n}}[ V^*_{\pi_n}]$ and
	\begin{small}
		$
		\epsilon^{\lambda}_{\text{class}} = \min_{\pi \in \Pi} \frac{1}{N} ( \sum_{n=1}^{N} \E_{d_{\pi_n}}  \E_{ \pi} [ (1-\lambda) Q_{\pi_n} + \lambda Q_{\pi^*}  ]) \\
		- \frac{1}{N}( \sum_{n=1}^{N} \E_{d_{\pi_n}}[ (1-\lambda) V^*_{\pi_n} + \lambda V_{\pi^*} ] )
		$.
	\end{small}
\end{theorem}
In fact, the performance  in Theorem~\ref{th:mix policy} is actually a lower bound of Theorem 3 in~\citep{chang2015learning}.\footnote{The main difference is due to technicalities. In~\cite{chang2015learning}, $\epsilon^{\lambda}_{\text{class}}$ is compared with a time-varying policy.}
Theorem~\ref{th:mix policy} says that on average $\pi_n$ has performance between the expert policy $J(\pi^*)$ and the intermediate cost $J_{\pi_n}^*$, as long as $\epsilon^{\lambda}_{\text{class}}$ is small (i.e., there exists a single policy in $\Pi$ that is better than the expert policy or the local improvement from any policy in $\Pi$). 
However, due to the presence of $\epsilon^{\lambda}_{\text{class}}$, despite $J_{\pi_n}^* \leq J(\pi_n)$, it is not guaranteed that $J_{\pi_n}^* \leq J(\pi^*)$. As in~\cite{chang2015learning}, either \lols or \slols can necessarily perform on average better than the expert policy $\pi^*$. 
Finally, we note that recently both~\citet{nair2017overcoming} and~\citet{rajeswaran2017learning} propose a scheme similar to \slols, but with the \textsc{AggreVaTe(D)} gradient computed using offline batch data collected by the expert policy. However, there is no theoretical analysis of this algorithm's performance.

\section{EXPERIMENTS}
We evaluate \loki on several robotic control tasks from OpenAI Gym~\citep{openaigym} with the DART physics engine~\citep{Lee2018}\footnote{The environments are defined in DartEnv, hosted at https://github.com/DartEnv.} and compare it with several baselines:
 \trpo~\citep{schulman2015trust}, \trpo from expert, 
 \daggered (the first-order version of \dagger~\citep{ross2011reduction} in~\eqref{eq:daggered gradient}),
 \slols (Section~\ref{sec:related work}), and
 \thor~\citep{sun2018truncated}.

\subsection{TASKS}
We consider the following tasks. In all tasks, the discount factor of the RL problem is set to $\gamma=0.99$. The details of each task are specified in Table~\ref{table:tasks} in Appendix~\ref{app:task details}. 

\paragraph{Inverted Pendulum} This is a classic control problem, and its goal is to swing up an pendulum and to keep it balanced in a upright posture. The difficulty of this task is that the pendulum cannot be swung up directly due to a torque limit.
\paragraph{Locomotion} 
The goal of these tasks (Hopper, 2D Walker, and 3D Walker) is to control a walker to 
move forward as quickly as possible without falling down.
In Hopper, the walker is a monoped, which is subjected to significant contact discontinuities, whereas the walkers in the other tasks are bipeds. 
In 2D Walker, the agent is constrained to a plane to simplify balancing.
\paragraph{Robot Manipulator} In the Reacher task, a 5-DOF (degrees-of-freedom) arm is controlled to reach a random target position in 3D space. The reward consists of the negative distance to the target point from the finger tip plus a control magnitude penalty. 
The actions correspond to the torques applied to the $5$ joints.
 
\subsection{ALGORITHMS}

We compare five algorithms (\loki, \trpo, \daggered, \thor, \slols) and the idealistic setup of performing policy gradient steps directly from the expert policy (\ideal).
To facilitate a fair comparison, all the algorithms are implemented based on a publicly available \trpo implementation~\citep{baselines}. Furthermore, they share the same parameters except for those that are unique to each algorithm as listed in Table~\ref{table:tasks} in Appendix~\ref{app:task details}. The experimental results averaged across $25$ random seeds are reported in Section~\ref{sec:experimental results}.

\paragraph{Policy and Value Networks}
Feed-forward neural networks are used to construct the policy networks and the value networks in all the tasks (both have two hidden layers and 32 tanh units per layer). 
We consider Gaussian stochastic policies, i.e. for any state $s \in \Sbb$, $\pi(a|s)$ is Gaussian distributed. 
The mean of the Gaussian $\pi(a|s)$, as a function of state, is modeled by the policy network, and the covariance matrix of Gaussian is restricted to be diagonal and independent of state. The policy networks and the value function networks are initialized randomly, except for the ideal setup (\trpo from expert), which is initialized as the expert.

\paragraph{Expert Policy} 
The same sub-optimal expert is used by all algorithms (\loki, \daggered, \slols, and \thor). It is obtained by running \trpo and stopping it before convergence. The estimate of the expert value function $V_{\pi^*}$ (required by \slols and \thor) is learned by minimizing the sum of squared TD(0) error on a large separately collected set of demonstrations of this expert. The final explained variance for all the tasks is more than $0.97$ (see Appendix~\ref{app:task details}).

\paragraph{First-Order Oracles}
The on-policy advantage $A_{\pi_n}$ in the first-order oracles for \trpo, \slols, and \loki (in the reinforcement phase) is implemented using an on-policy value function estimator and Generalized Advantage Estimator (GAE)~\citep{schulman2015high}.
For \daggered and the imitation phase of \loki, the first-order oracle is calculated using~\eqref{eq:KL divergence}.
For \slols, we use the estimate $A_{\pi^*}(s_t, a_t) \approx  c(s_t,a_t) + \gamma \hat{V}_{\pi^*}(s_{t+1}) - \hat{V}_{\pi^*}(s_t)$. 
And for \thor, $A_{\pi_n, t}^{H, \pi^*}$ of the truncated-horizon problem is approximated by Monte-Carlo samples with an on-policy value function baseline estimated by regressing on these Monte-Carlo samples.
Therefore, for all methods, an on-policy component is used in constructing the first-order oracle. 
The exponential weighting in GAE is $0.98$; the mixing coefficient $\lambda$ in \slols is $0.5$; $N_M$ in \loki is reported in Table~\ref{table:tasks} in Appendix A, and $N_m = \floor*{\frac{1}{2}N_M}$, and $d=3$.

\paragraph{Mirror Descent}
After receiving an update direction $g_n$ from the first-order oracle, a KL-divergence-based trust region is specified. This is equivalent to setting the strictly convex function $R_n$ in mirror descent to $\frac{1}{2} \theta^\top F(\theta_n) \theta$ and choosing a proper learning rate.
In our experiments, a larger KL-divergence limit ($0.1$) is selected for imitation gradient~\eqref{eq:KL divergence} (in \daggered and in the imitation phase of \loki), and a smaller one ($0.01$) is set for all other algorithms.
This decision follows the guideline provided by the theoretical analysis in Section~\ref{sec:imitation gradient description} and is because of the low variance in calculating the gradient of~\eqref{eq:KL divergence}.
Empirically, we observe using the larger KL-divergence limit with policy gradient led to high variance and instability.

\begin{figure*}[h!]
	\captionsetup[subfloat]{farskip=0pt,captionskip=0pt}
	\centering
	\subfloat{\includegraphics[width = \linewidth/3]{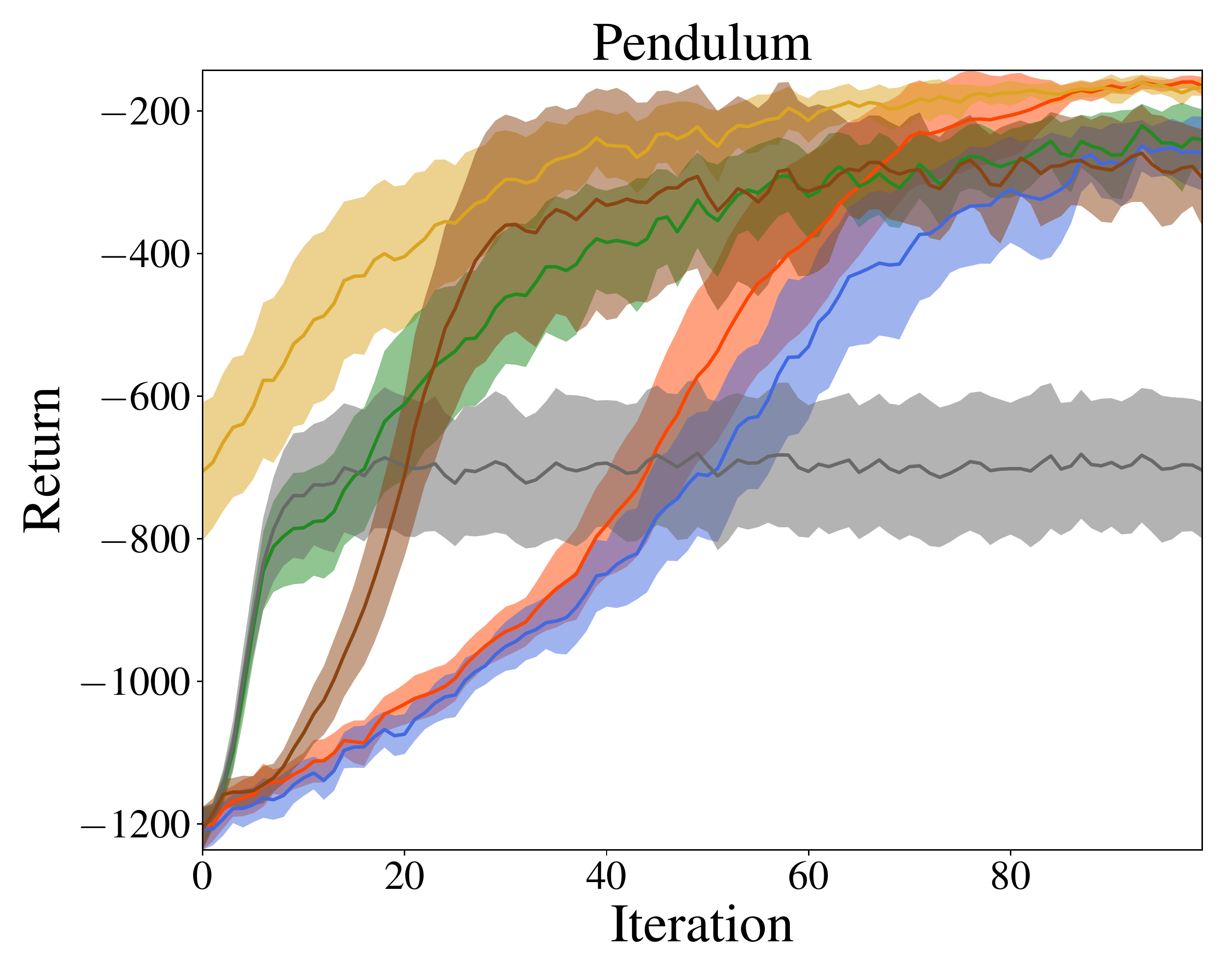}} \hfill
	\subfloat{\includegraphics[width = \linewidth/3]{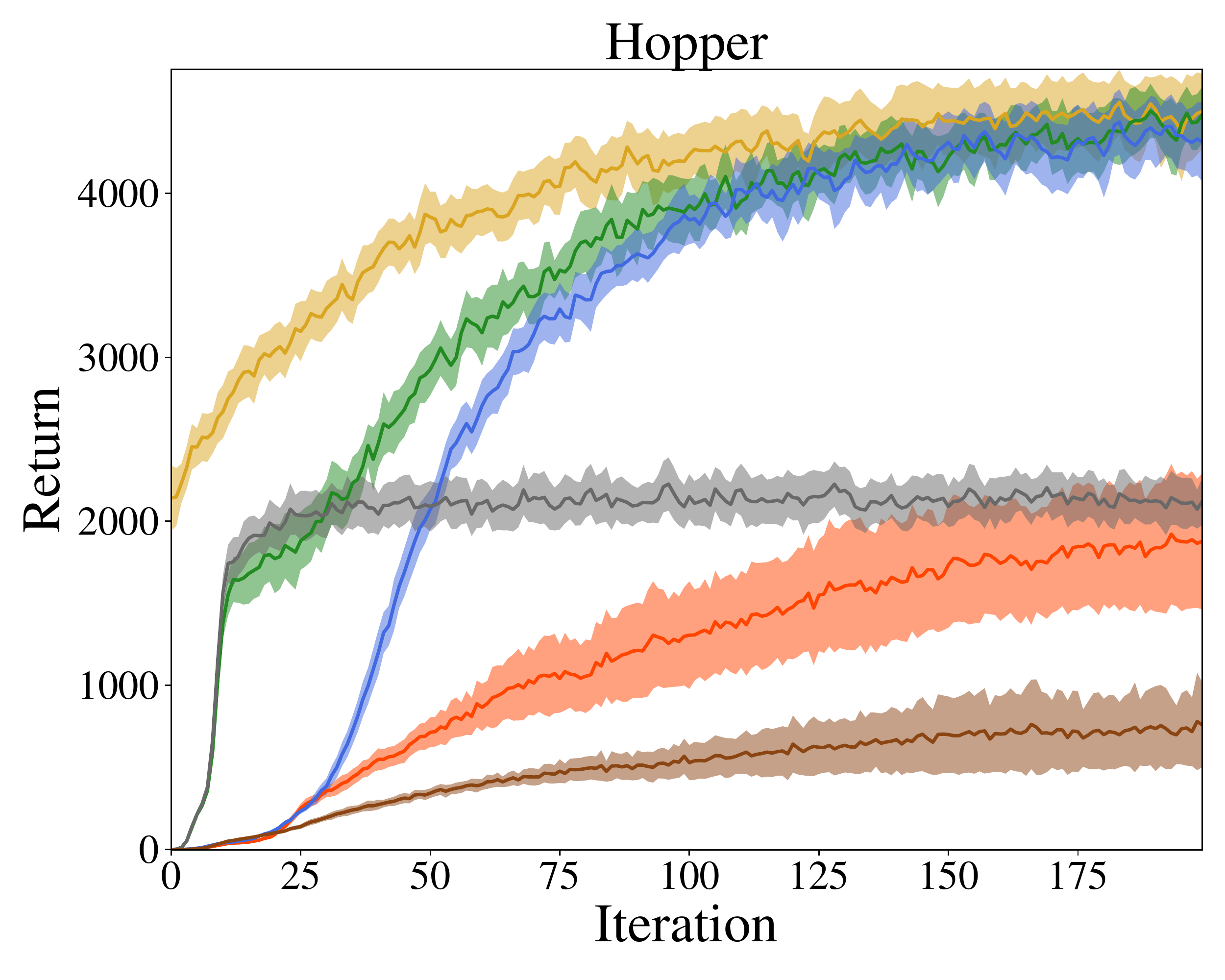}} \hfill 
	\subfloat{\includegraphics[width = \linewidth/3]{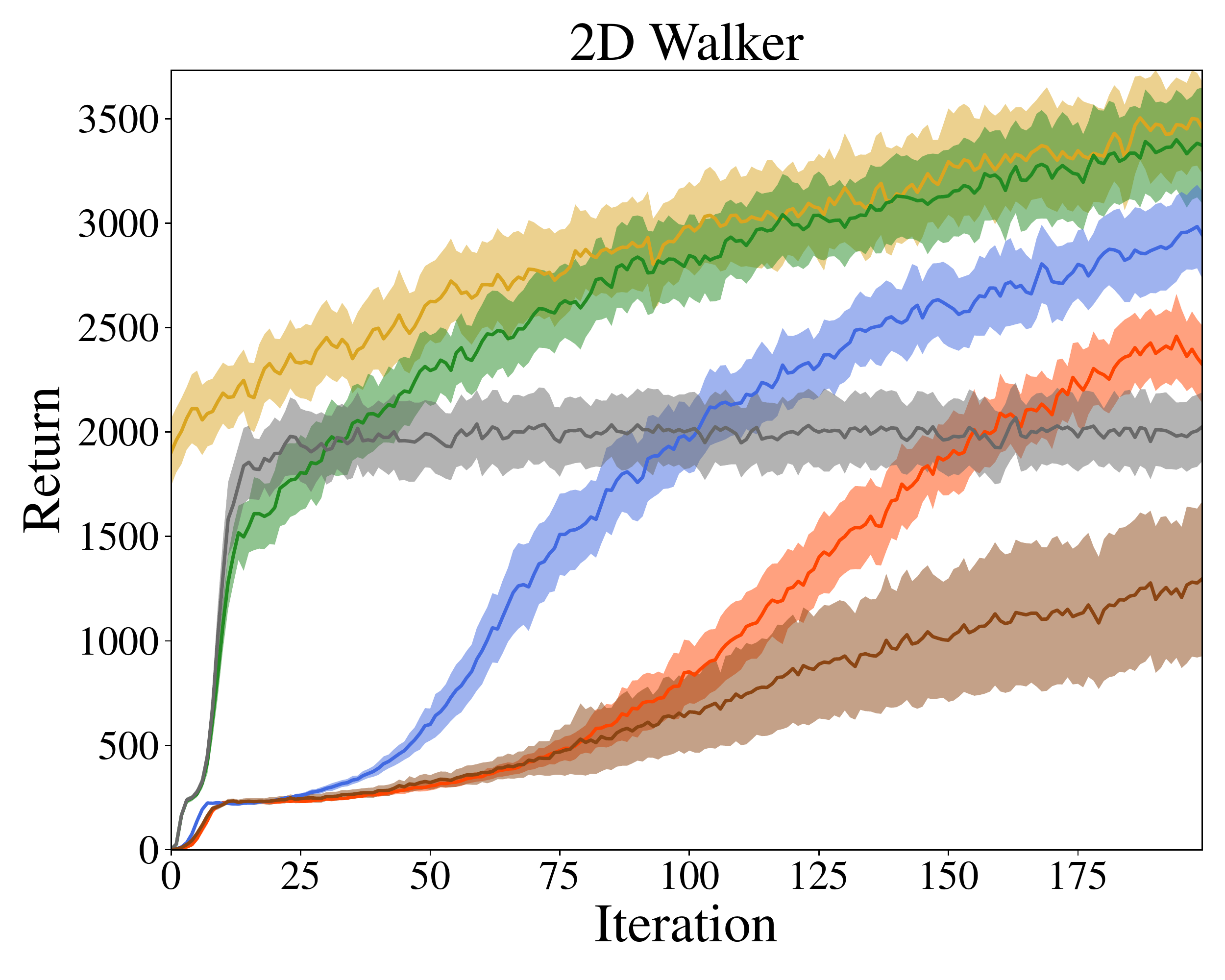}} \\ 
	\subfloat{\includegraphics[width = \linewidth/3]{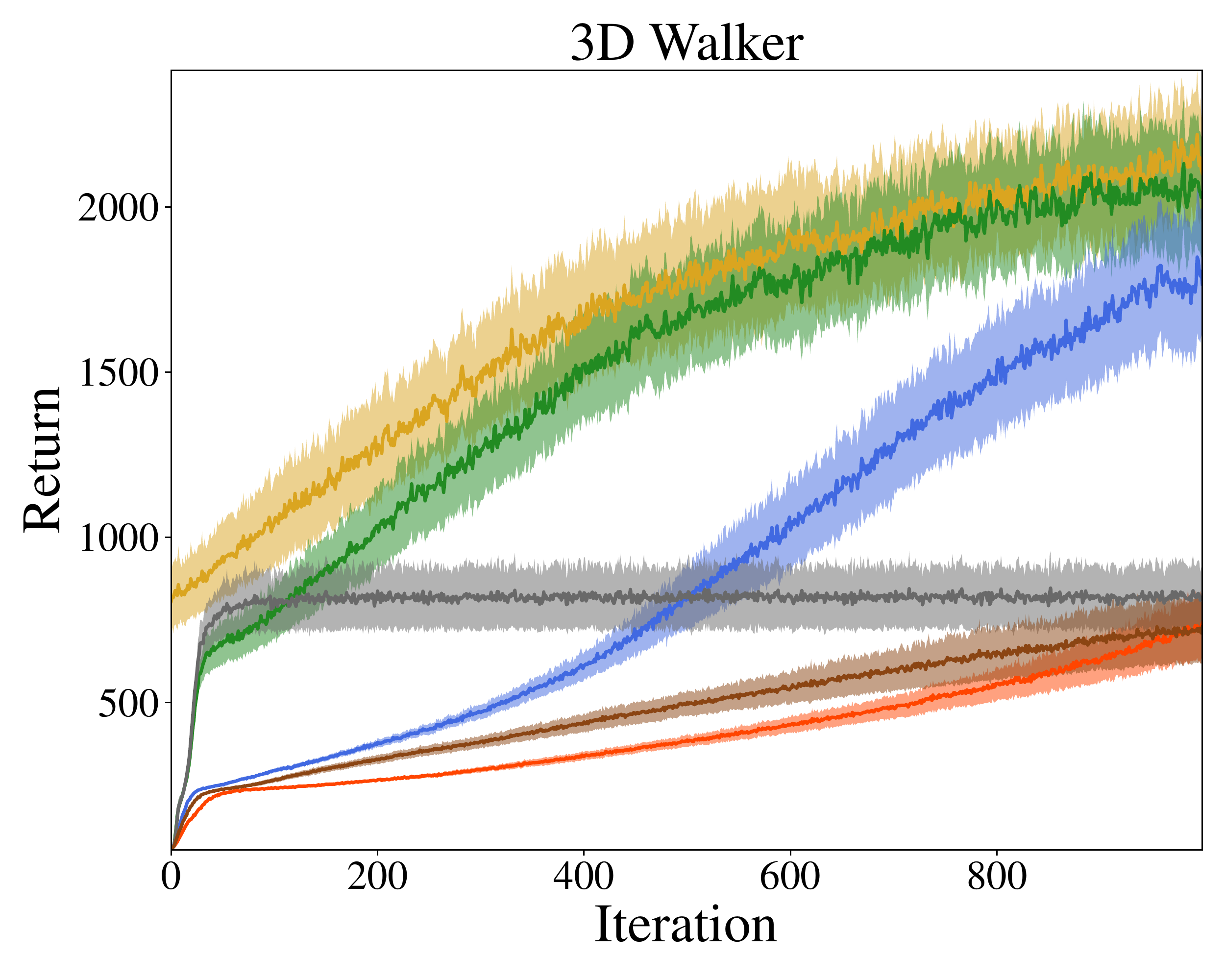}} \hfill
    \subfloat{\includegraphics[width = \linewidth/3]{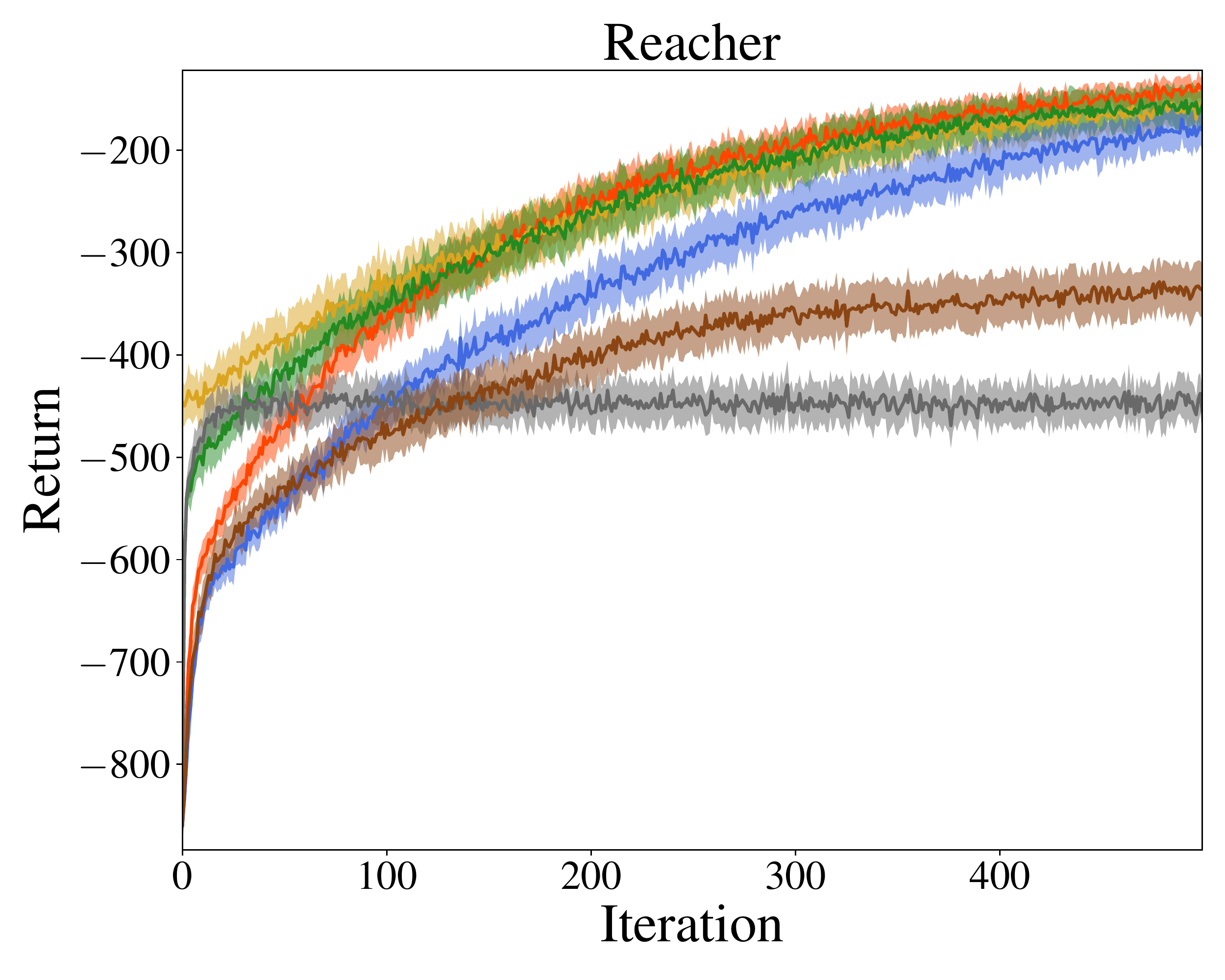}} \hfill
    \subfloat{\includegraphics[width = \linewidth/3]{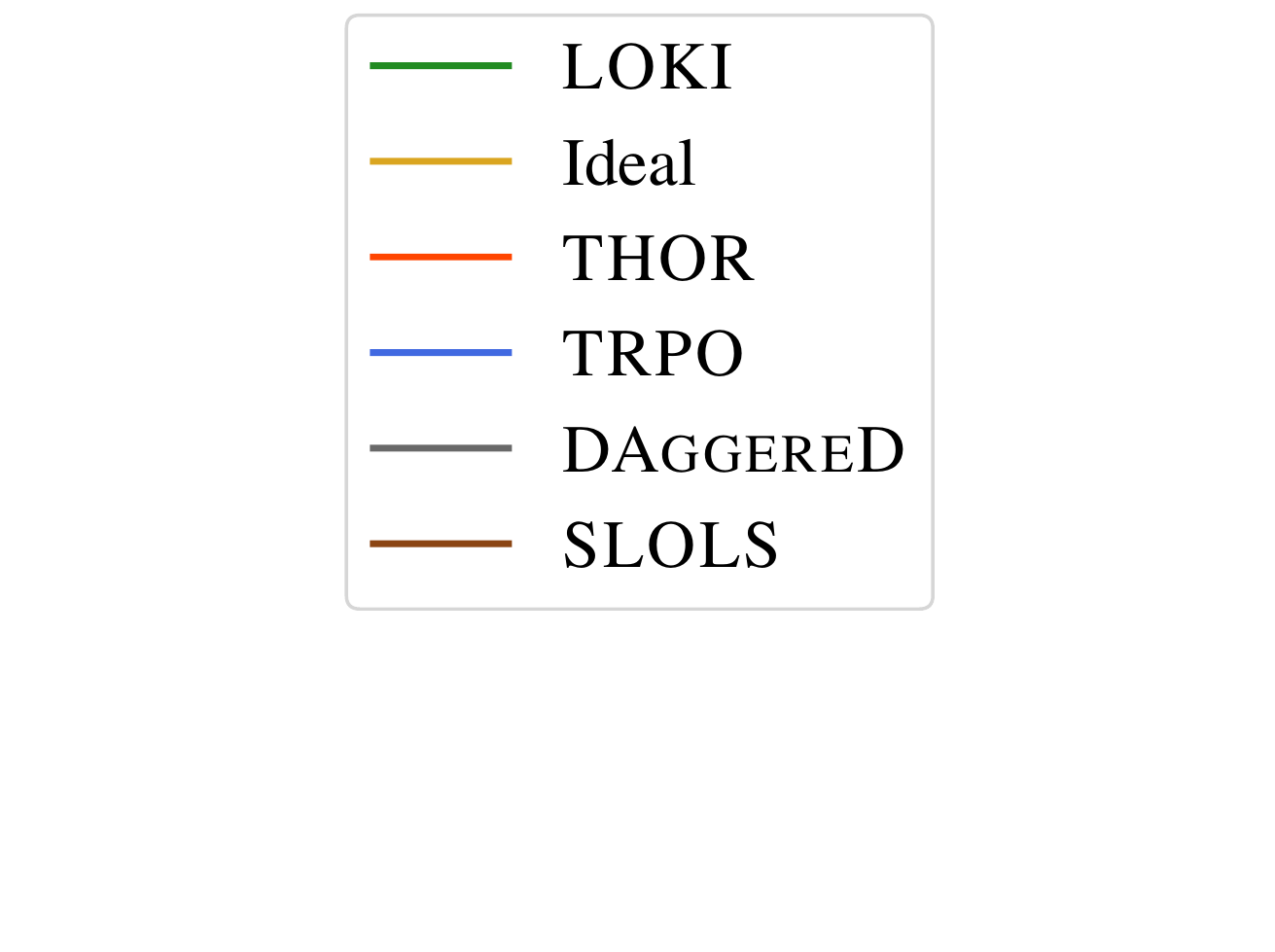}}
	\caption{Learning curves. Shaded regions correspond to $\pm \frac{1}{2}$-standard deviation.}
	\label{fig:experiments}
\end{figure*}

\subsection{EXPERIMENTAL RESULTS} \label{sec:experimental results}
We report the performance of these algorithms on various tasks in Figure~\ref{fig:experiments}.
The performance is measured by the accumulated \emph{rewards}, which are directly provided by OpenAI Gym.

We first establish the performance of two baselines, which represent standard RL (\trpo) and standard IL (\daggered). \trpo is able to achieve considerable and almost monotonic improvement from a randomly initialized policy.  \daggered reaches the performance of the suboptimal policy in a relatively very small number of iterations, e.g. 15 iterations in 2D Walker, in which  the suboptimal policy to imitate is \trpo at iteration 100. However, it fails to outperform the suboptimal expert.

Then, we evaluate the proposed algorithm \loki and Ideal, the performance of which we wish to achieve in theory.
\loki consistently enjoys the best of both \trpo and \daggered: it improves as fast as \daggered at the beginning, keeps improving, and then finally matches the performance of Ideal after transitioning into the reinforcement phase. 
Interestingly, the on-policy value function learned, though not used, in the imitation phase helps \loki transition from  imitation phase to reinforcement phase smoothly.

Lastly, we compare \loki to the two other baselines (\slols and \thor) that combine RL and IL. 
\loki outperforms these two baselines by a considerably large margin in Hopper, 2D Walker, and 3D Walker; 
but surprisingly, the performance of \slols and \thor are inferior even to \trpo on these tasks.
The main reason is that the first-order oracles of both methods is based on an estimated expert value function $\hat{V}_{\pi^*}$. As $\hat{V}_{\pi^*}$ is only regressed on the data collected by running the expert policy, large covariate shift error could happen if the dimension of the state and action spaces are high, or if the uncontrolled system is complex or unstable.
For example, in the low-dimensional Pendulum task and the  simple Reacher task,
the expert value function can generalize better. 
As a result, in these two cases, \lols and \thor achieve super-expert performance.
However, in more complex tasks, where the effects of covariant shift amplifies exponentially with the dimension of the state space,  \thor and \slols start to suffer from the inaccuracy of $\hat{V}_{\pi^*}$, as illustrated in the 2D Walker and 3D Walker tasks.

\section{CONCLUSION}

We present a simple, elegant algorithm, \loki, that combines the best properties of RL and IL. Theoretically, we show that, by randomizing the switching time, \loki can perform as if running policy gradient steps directly from the expert policy. Empirically, \loki demonstrates superior performance compared with the expert policy and more complicated algorithms that attempt to combine RL and IL. 

\subsubsection*{Acknowledgements}
This work was supported in part by NSF NRI Award
1637758, NSF CAREER Award 1750483, and NSF Graduate Research Fellowship under Grant No. 2015207631.

{\renewcommand{\section}[2]{\subsubsection*{References}}
\bibliographystyle{apalike}
\bibliography{ref}
}

\clearpage

\onecolumn
\appendix
\section{Task Details} \label{app:task details}

\begin{table*}[h]
	\label{table:tasks}
	\begin{center}
		\begin{tabular}{lccccc}
			&   Pendulum &  Hopper & 2D Walker & 3D Walker & Reacher\\
			\hline \\
			Observation space dimension   						&	3  	&	11	&	17	&	41	&	21	\\
			Action space dimension        						& 	1	&	3	&	6	&	15	&	5	\\
			Number of samples per iteration 					&	4k 	&	16k	&	16k	&	16k	&	40k	\\
			Number of iterations     							&	100	&	200	&	200 &	1000&	500 \\
            Number of \trpo iterations for expert 				&	50	&	50	&	100	&	500 &	100	\\     
			Upper limit of number of imitation steps of \loki	&	10	&	20	&	25	&	50	&	25	\\	
			Truncated horizon of \thor          				&	40	&	40	&	250 &	250 &	250 \\
		\end{tabular}
	\end{center}
\end{table*}
The expert value estimator $\hat V_{\pi^*}$ needed by \slols and \thor were trained on a large set of samples 
(50 times the number of samples used in each batch in the later policy learning),
and the final average TD error are: Pendulum ($0.972$), Hopper ($0.989$), 2D Walker ($0.975$), 3D Walker ($0.983$), and Reacher ($0.973$), measured in terms of explained variance, which is defined as 1- (variance of error /  variance of prediction). 

\section{Proof of Section~\ref{sec:comparison}} \label{app:proof for comparison}

\subsection{Proof of Proposition~\ref{pr:nonconvex performance}} \label{app:proof of th:nonconvex performance}

To prove Proposition~\ref{pr:nonconvex performance}, we first prove a useful Lemma~\ref{lm:mirror descent improvement}. 
\begin{lemma} \label{lm:mirror descent improvement}
	Let $\KK$ be a convex set. Let $h= \E[g]$. Suppose $R$ is $\alpha$-strongly convex with respect to norm $\norm{\cdot}$.
	\begin{align*}
	y =  \argmin_{z \in \KK} \lr{g}{z} + \frac{1}{\eta}D_{R}(z||x)
	\eqqcolon P_{g,\eta}(x)
	\end{align*}
	where $\eta$ satisfies that 
	$
	-\alpha\eta + \frac{\beta \eta^2}{2} \leq 0 
	$.
	Then it holds
	\begin{align*}
	&\E[\lr{h}{y - x} +  \frac{\beta}{2} \norm{x-y}^2] \leq \frac{1}{2} \left( -\alpha\eta + \frac{\beta \eta^2}{2} \right)\E\left[ \norm{H}^2 \right] + \frac{2 \eta}{\alpha}  \E\left[ \norm{g-h}_*^2 \right] 
	\end{align*}
	where  $H = \frac{1}{\eta}(x -  P_{h, \eta}(x))$.
	In particular,  if $\norm{\cdot} = \norm{\cdot}_W$ for some positive definite matrix $W$, $R$ is quadratic, and $\KK$ is Euclidean space, 
	\begin{align*}
	&\E[\lr{h}{y - x} +  \frac{\beta}{2} \norm{x-y}^2] \leq\left( -\alpha\eta + \frac{\beta \eta^2}{2} \right)\E\left[ \norm{H}^2 \right] + \frac{\beta \eta^2}{2} \E[\norm{H-G}^2]
	\end{align*}
\end{lemma}
\begin{proof}
	Let  $G = \frac{1}{\eta}(x -P_{g,\eta}(x))$. First we show for the special case (i.e. suppose $R(x) = \frac{1}{2} \lr{x}{M x}$ for some positive definite matrix $M$, and therefore $G = M^{-1}g$ and $H = M^{-1} h$). 
	\begin{align*}
	\E[\lr{h}{y - x}] = - \eta \E[\lr{h}{G} ] 
	= - \eta\E[\lr{h}{H} ] \leq - \alpha \eta \norm{H}^2
	\end{align*}
	and because $g$ is unbiased, 
	\begin{align*}
 	\E\left[\frac{\beta}{2}\norm{x-y}^2\right] =  \E\left[\frac{\eta^2 \beta}{2}\norm{H}^2 + \frac{\eta^2 \beta}{2}\norm{G-H}^2 \right]
	\end{align*}
	For general setups, we first separate the term into two parts		
	\begin{align*}
	\lr{h}{y - x} = \lr{g}{y - x} +  \lr{h-g}{y - x}
	\end{align*}
	For the first term, we use the optimality condition
	\begin{align*}
	\lr{g + \frac{1}{\eta}\nabla R(y) - \frac{1}{\eta} \nabla R(x)}{z - y} \geq 0, \quad \forall z \in \KK
	\end{align*}
	which implies
	\begin{align*}
	\lr{g }{x - y} \geq \frac{1}{\eta} \lr{\nabla R(y) -  \nabla R(x)}{y - x} \geq \frac{\alpha}{\eta} \norm{x-y}^2
	\end{align*}
	Therefore, we can bound the first term by
	\begin{align*}
	\lr{g}{y - x}  &\leq - \frac{\alpha}{\eta} \norm{x-y}^2 
	=-\alpha\eta \norm{G}^2 
	\end{align*}
	
	On the other hand, for the second term, we first write
	\begin{align*}
	\lr{h-g}{y - x} &= -\eta\lr{h-g}{G} \\
	& =  -\eta\lr{h-g}{H} + \eta\lr{h-g}{H-G}
	\end{align*}
	and we show that 
	\begin{align} \label{eq:continuity of prox map} 
	\lr{h-g}{H-G} \leq \norm{h-g}_*\norm{H-G} \leq \frac{\norm{h-g}_*^2 }{\alpha}
	\end{align}
	This can be proved by Legendre transform:
	\begin{align*}
	P_{g,\eta}(x) = &\argmin_{z \in \KK} \lr{g}{z} + \frac{1}{\eta}D_{R}(z||x) \\
	&=\argmin_{z \in \KK} \lr{g - \frac{1}{\eta}\nabla R(x)}{z} +\frac{1}{\eta} R(z) \\
	&=   \nabla \left(\frac{1}{\eta} R\right)^* \left(\frac{1}{\eta} \nabla R(x) - g\right)
	\end{align*}
	Because $\frac{1}{\eta} R$ is $\frac{\alpha}{\eta}$-strongly convex with respect to norm $\norm{\cdot}$, $\left(\frac{1}{\eta} R\right)^*$ is $\frac{\eta}{\alpha}$-smooth with respect to norm $\norm{\cdot}_*$, we have
	\begin{align*}
	\norm{H-G} \leq \frac{1}{\eta}\frac{\eta }{\alpha} \norm{g - h}_* = \frac{1 }{\alpha} \norm{g - h}_*
	\end{align*}
	which proves~\eqref{eq:continuity of prox map}.
	Putting everything together, we have 
	\begin{align*}
	&\E[\lr{h}{y - x} +  \frac{\beta}{2} \norm{x-y}^2] \\
	&\leq \E\left[ \left( -\alpha\eta + \frac{\beta \eta^2}{2} \right) \norm{G}^2  \right] 
	 + \E\left[ -\eta\lr{h-g}{H}+  \frac{\eta}{\alpha}  \norm{g-h}_*^2 \right] \\
	&=  \E\left[ \left( -\alpha\eta + \frac{\beta \eta^2}{2} \right) \norm{G}^2  \right] + \E\left[   \frac{\eta}{\alpha}  \norm{g-h}_*^2 \right] 
	\end{align*}
	Because 
	\begin{align*}
	\norm{H}^2 \leq 2\norm{G}^2 + 2 \norm{H-G}^2 \leq 2\norm{G}^2 + \frac{2}{\alpha^2} \norm{h-g}_*^2
	\end{align*}
	it holds that 
	\begin{align*}
	&\E[\lr{h}{y - x} +  \frac{\beta}{2} \norm{x-y}^2] \\
	&\leq \frac{1}{2} \left( -\alpha\eta + \frac{\beta \eta^2}{2} \right)\E\left[ \norm{H}^2 \right] + \frac{ 2 \eta}{\alpha}  \E\left[ \norm{g-h}_*^2 \right] 
	\end{align*}
\end{proof}

\paragraph{Proof of Proposition~\ref{pr:nonconvex performance}}
We apply Lemma~\ref{lm:mirror descent improvement}: By smoothness of $J$, 
\begin{align*}
\E\left[ J(\pi_{n+1}) \right] - J(\pi_n) &\leq  \E\left[ \lr{\nabla J(\pi_n)}{\theta_{n+1} - \theta_n}
+ \frac{\beta}{2} \norm{\theta_{n+1} - \theta_n}^2 \right] \\
&\leq  \frac{1}{2}\left(  - \alpha_n \eta_n +  \frac{ \beta \eta_n^2}{2} \right) \E\left[ \norm{\hat{\nabla}_\theta J(\pi_n) }^2 \right]+ \frac{2 \eta_n  }{\alpha_n}    \norm{\nabla_\theta J(\pi_n) - g_n}_*^2
\end{align*}
This proves the statement in Proposition~\ref{pr:nonconvex performance}. We note that, in the above step, the general result of Lemma~\ref{lm:mirror descent improvement}. For the special case~Lemma~\ref{lm:mirror descent improvement}, we would recover the usual convergence property of stochastic smooth nonconvex optimization, which shows on average convergence to stationary points in expectation.

\subsection{Proof of Proposition~\ref{pr:online convex performance}}

We use a well-know result of mirror descent, whose proof can be found e.g. in~\citep{juditsky2011first}.
\begin{lemma} \label{lm:mirror descent}
Let $\KK$ be a convex set. Suppose $R$ is $\alpha$-strongly convex with respect to norm $\norm{\cdot}$. Let $g$ be a vector in some Euclidean space and let
\begin{align*}
y =  \argmin_{z \in \KK} \lr{g}{z} + \frac{1}{\eta}D_{R}(z||x)
= P_{g,\eta}(x)
\end{align*}
Then for all $z \in \KK$
\begin{align*}
\eta \lr{g}{x - z} \leq D_R(z||x) - D_R(z||y) + \frac{\eta^2}{2} \norm{g}_*^2
\end{align*}
\end{lemma}

Next we prove a lemma of performing online mirror descent with weighted cost. While weighting it not required in proving Proposition~\ref{pr:online convex performance}, it will be useful to prove Theorem~\ref{th:switch} later in Appendix~\ref{app:proof of method}.
\begin{lemma} \label{lm:weighted online mirror descent}
	Let $f_n$ be $\sigma$-strongly convex with respect to some strictly convex function $R_n$, i.e.
	\begin{align*}
	f_n(x) \geq f_n(y) + \lr{\nabla f_n(y)}{x-y} + \sigma D_{R_n}(x||y)
	\end{align*} 
	and let $\{w_n\}_{n=1}^N$ be a sequence of positive numbers. 
	Consider the update rule
	\begin{align*}
	x_{n+1} =  \argmin_{z \in \KK} \lr{w_n g_n}{x} + \frac{1}{\eta_n}D_{R_n}(z||x_n)
	\end{align*}
	 where $g_n = \nabla f_n (x_n)$ and  
	$	\eta_n = \frac{1}{ \hat{\sigma} \sum_{m=1}^{n} w_m}$. 
	Suppose $\hat{\sigma} \leq \sigma $. Then for all $x^* \in \KK$, $N \geq M \geq 1$, it holds that 
	\begin{align*}
	\sum_{n=M}^{N} w_n f_n(x_n) - w_n f_n(x^*) 
	&\leq  \hat{\sigma} D_{R_M}(x^*|| x_M) \sum_{n=1}^{M-1}w_n +  \frac{1}{2\hat{\sigma}} \sum_{n=1}^{N}\frac{ w_n^2 \norm{g_n}_*^2}{ \sum_{m=1}^{n} w_m } 
	\end{align*}
\end{lemma}
\begin{proof} \allowdisplaybreaks
	The proof is straight forward by strong convexity of $f_n$ and Lemma~\ref{lm:mirror descent}. 	
	\begin{align*}
	&\sum_{n=M}^{N} w_n (f_n(x_n) -f_n(x^*)) \\
	&\leq \sum_{n=M}^{N} w_n \left( \lr{ g_n}{x_n - x^*} -  \sigma D_{R_n}(x^*||x_n) \right)  \qquad\text{($\sigma$-strong convexity)}  \\
	&\leq  \sum_{n=M}^{N} \frac{1}{\eta_n} D_{R_n}(x^*||x_n) - \frac{1}{\eta_n}  D_{R_n}(x^*||x_{n+1}) - w_n \sigma D_{R_n}(x^*||x_n)  + \frac{w_n^2 \eta_n}{2} \norm{g_n}_*^2  \qquad\text{(Lemma~\ref{lm:mirror descent})} \\
	&\leq  \frac{D_{R_M}(x^*|| x_M)}{\eta_{M-1}} + \sum_{n=M}^{N} \left( \frac{1}{\eta_n} -  \frac{1}{\eta_{n-1}}  - w_n \sigma \right) D_{R_n}(x^*||x_n) + \frac{ w_n^2 \eta_n}{2} \norm{g_n}_*^2  \qquad\text{(We define $\frac{1}{\eta_0}  =0 $)}\\
	&=  \hat{\sigma} D_{R_M}(x^*|| x_M) \sum_{n=1}^{M-1}w_n  + \sum_{n=1}^{N} \left( w_n \hat{\sigma} - w_n \sigma \right) D_{R_n}(x^*||x_n) +  \frac{1}{2\hat{\sigma}} \sum_{n=1}^{N}\frac{ w_n^2 \norm{g_n}_*^2}{ \sum_{m=1}^{n} w_m } \\
	&\leq   \hat{\sigma} D_{R_M}(x^*|| x_M) \sum_{n=1}^{M-1}w_n +  \frac{1}{2\hat{\sigma}} \sum_{n=1}^{N}\frac{ w_n^2 \norm{g_n}_*^2}{ \sum_{m=1}^{n} w_m }  \qedhere
	\end{align*}
\end{proof}

\paragraph{Proof of Proposition~\ref{pr:online convex performance}}
Now we use Lemma~\ref{lm:weighted online mirror descent} to prove the final result. 
It's easy to see that if $g_n$ is an unbiased stochastic estimate of $ \nabla f_n (x_n)$ in Lemma~\ref{lm:weighted online mirror descent}, then the performance bound would hold in expectation since $x_n$ does not depend on $g_n$.
Finally, by definition of $\epsilon_{\text{class}}$, this concludes the proof.

\section{Proof of Section~\ref{sec:method}} \label{app:proof of method}

\subsection{Proof of Theorem~\ref{th:weighted random stop}}
Let $w_n = n^d $. The proof is similar to the proof of Proposition~\ref{pr:online convex performance} but with weighted cost. 
First we use  Lemma~\ref{lm:performance difference} and bound the series of weighted accumulated loss
\begin{align*}
\E\left[ \sum_{n=N_m}^{N_{M}}w_n J(\pi_n) \right] - \left( \sum_{n=N_m}^{N_{M}} w_n \right) J(\pi^*)  \leq \frac{C_{\pi^*}}{1-\gamma  } \sum_{n=N_m}^{N_M} w_n l_n(\pi_n)
\end{align*}
Then we bound the right-hand side by using Lemma~\ref{lm:weighted online mirror descent},
\begin{align*}
\sum_{n=N_m}^{N_M} w_n l_n(\pi_n) -\min_{\pi \in \Pi} \sum_{n=N_m}^{N_{M}} w_n l_n(\pi) 
&\leq   \hat{\sigma}  D_{\RR} \sum_{n=1}^{N_m-1}w_n   + \frac{1}{2\hat{\sigma}} \sum_{n=N_m}^{N_{M}}\frac{ w_n^2 \norm{g_n}_*^2}{ \sum_{m=1}^{n} w_m }   \\
&\leq   \frac{\hat{\sigma} D_{\RR}  N_m^{d+1}}{d+1}
+  \frac{d+1}{2\hat{\sigma}} \sum_{n=N_m}^{N_{M}} \norm{g_n}_*^2   n^{d-1}
\end{align*}
where we use the fact that   $d \geq 0$,
\begin{align*}
\frac{n^{d+1} - (m-1)^{d+1} }{d+1}  \leq \sum_{k=m}^{n} k^d \leq \frac{(n+1)^{d+1} -m^{d+1}}{d+1}
\end{align*}
which implies $
  \frac{w_n^2}{\sum_{m=1}^{n} w_m} \leq \frac{(d+1) n^{2d} }{n^{d+1}} \leq (d+1) n^{d-1}.
$
Combining these two steps, we see that  the weighted accumulated loss on average can be bounded by
\begin{align*}
\E\left[ \frac{\sum_{n=N_m}^{N_M} w_n J(\pi_n)}{\sum_{n=N_m}^{N_M} w_n} \right] 
&\leq   J(\pi^*) +  \frac{C_{\pi^*}}{1-\gamma } \left(  \epsilon^w_{\text{class}}+
		\frac{\hat{\sigma} D_{\RR}  N_m^{d+1}}{(d+1) \sum_{n=N_m}^{N_M} w_n }
		+  \frac{d+1}{2\hat{\sigma} \sum_{n=N_m}^{N_M} w_n } \sum_{n=N_m}^{N_{M}} \norm{g_n}_*^2   n^{d-1}  \right) \\ 		
\end{align*}
Because $N_M \geq 2 N_m$ and $\frac{x}{1-x}\leq 2 x$ for $x \leq \frac{1}{2}$, we have
\begin{align*}
	\frac{  N_m^{d+1}}{(d+1) \sum_{n=N_m}^{N_M} w_n } \leq  \frac{  N_m^{d+1}}{  N_M^{d+1}  - (N_{m}-1)^{d+1}}  
	\leq \frac{  N_m^{d+1}}{  N_M^{d+1}  - N_{m}^{d+1}}  = \frac{1}{  \left(\frac{N_M}{N_m}\right)^{d+1} -1}
	\leq  2	 \left( \frac{N_m}{N_M} \right)^{d+1} \leq  2^{-d}
\end{align*}
and, for $d\geq 1$,
\begin{align*}
 \frac{d+1}{\sum_{n=N_m}^{N_M} w_n } \sum_{n=N_m}^{N_{M}}  n^{d-1} &\leq  \frac{d+1}{N_M^{d+1}  - (N_{m}-1)^{d+1} } \frac{d+1}{d} \left( (N_M+1)^d - N_m^d  \right)\\
 &\leq \frac{(d+1)^2}{d} \frac{(N_M+1)^d}{N_M^{d+1}  - N_{m}^{d+1} } \\
  &\leq \frac{(d+1)^2}{d} \frac{\frac{1}{N_M}(1+\frac{1}{N_M})^d}{1  - \left(\frac{N_m}{N_M}\right)^{d+1} }  \\
    &\leq   \frac{16d}{3N_M}\left( 1+\frac{1}{N_M} \right)^d  \qquad\text{($N_M\geq 2 N_m$ and $d\geq1$) } \\
    &\leq  \frac{16d}{3 N_M} \exp\left( \frac{d}{N_M} \right)
\end{align*}
and for $d=0$, 
\begin{align*}
\frac{d+1}{\sum_{n=N_m}^{N_M} w_n } \sum_{n=N_m}^{N_{M}}  n^{d-1}
= \frac{1}{\sum_{n=N_m}^{N_M} 1 } \sum_{n=N_m}^{N_{M}} \frac{1}{n}
\leq \frac{ \log(N_M)+1}{N_M - N_m} \leq \frac{ 2 \left( \log(N_M)+1\right) }{N_M}
\end{align*}
Thus, by the assumption that $\norm{g_n}_* \leq G $ almost surely, the weighted accumulated loss on average has an upper bound
\begin{align*}
\E\left[ \frac{\sum_{n=N_m}^{N_M} w_n J(\pi_n)}{\sum_{n=N_m}^{N_M} w_n} \right] 
&\leq   J(\pi^*) +  \frac{C_{\pi^*}}{1-\gamma } \left(  \epsilon^w_{\text{class}}+
2^{-d}	\hat{\sigma} D_{\RR} 
+  \frac{G^2 C_{N_M}/\hat{\sigma}}{N_M }   \right)  		
\end{align*}
By sampling $K$ according to $w_s$, this bound directly translates into the the bound on $J(\pi_{K})$.

\subsection{Proof of Theorem \ref{th:mix policy}}

For simplicity, we prove the result of deterministic problems. For stochastic problems, the result can be extended to expected performance, similar to the proof of Proposition~\ref{pr:online convex performance}. 
We first define the online learning problem of applying $g_n =  \nabla_{\theta} l_n^{\lambda}(\pi) |_{\pi = \pi_n}$ to update the policy. In the $n$th iteration, we define the per-round cost as 
\begin{align}
l_n^{\lambda}(\pi) = \E_{d_{\pi_n}} \E_{ \pi}[ (1-\lambda) A_{\pi_n} + \lambda A_{\pi^*}  ]
\end{align}
With the strongly convexity assumption and large enough step size, similar to the proof for Proposition~\ref{pr:online convex performance}, we can show that 
\begin{align*}
\sum_{n=1}^{N} l_n^{\lambda}(\pi_n) &\leq \min_{\pi \in \Pi}  \sum_{n=1}^{N} (l_n^{\lambda}(\pi) + \epsilon^\lambda_{\text{regret}}) \\
& =   \min_{\pi \in \Pi}   \sum_{n=1}^{N} \E_{d_{\pi_n}}  \E_{ \pi} [ (1-\lambda) A_{\pi_n} + \lambda A_{\pi^*}  ]  + N\epsilon^\lambda_{\text{regret}} 
\end{align*}
where $\epsilon^\lambda_{\text{regret}} = \tilde{O}\left(\frac{1}{T}\right)$. Note by definition of $A_{\pi_n}$, the left-hand-side in the above bound can be written as  
\begin{align}
\frac{1}{N}\sum_{n=1}^{N} l_n^{\lambda}(\pi_n) = 
\sum_{n=1}^{N} \E_{d_{\pi_n}} \E_{\pi_n}[ (1-\lambda) A_{\pi_n} + \lambda A_{\pi^*}  ] = \sum_{n=1}^{N} \lambda \E_{d_{\pi_n}} \E_{\pi_n}[  A_{\pi^*}  ]
\end{align}

To relate this to the performance bound, we invoke Lemma~\ref{lm:performance difference} and write 
\begin{align*}
& \sum_{n=1}^{N}  J(\pi_n)  - \left(  (1-\lambda)J^*_n  + \lambda J(\pi^*) \right) \\
&=   \sum_{n=1}^{N} (1-\lambda) \left( J(\pi_n) - J^*_n\right)   + \frac{1 }{1-\gamma} \sum_{n=1}^{N}\lambda  \E_{d_{\pi_n}} \E_{\pi_n}[  A_{\pi^*}  ] \\
&\leq \sum_{n=1}^{N} (1-\lambda) \left( J(\pi_n) -J^*_n)\right)  + \min_{\pi \in \Pi}  \frac{1}{1-\gamma} \sum_{n=1}^{N} \E_{d_{\pi_n}}  \E_{ \pi} [ (1-\lambda) A_{\pi_n} + \lambda A_{\pi^*}  ]  +   \frac{N}{1-\gamma} \epsilon^\lambda_{\text{regret}}   \\
&=    \min_{\pi \in \Pi} \frac{1}{1-\gamma} \sum_{n=1}^{N} \E_{d_{\pi_n}}  \E_{ \pi} [ (1-\lambda) Q_{\pi_n} + \lambda A_{\pi^*}  ] + \frac{N}{1-\gamma} \epsilon^\lambda_{\text{regret}} 
  +  \sum_{n=1}^{N} \lambda\left(  - \frac{ \E_{d_{\pi_n}}  \E_{ \pi_n}  [ Q_{\pi_n}]}{1-\gamma}    +   J(\pi_n) - J^*_n \right) \\ %\qquad \text{(Definition of $A_{\pi_n}$) }\\
&=    \min_{\pi \in \Pi}  \frac{1}{1-\gamma}  \sum_{n=1}^{N} \E_{d_{\pi_n}}  \E_{ \pi} [ (1-\lambda) (Q_{\pi_n} - {V}^*_{\pi_n}) + \lambda( Q_{\pi^*} - V_{\pi^*})  ] + \frac{N}{1-\gamma} \epsilon^\lambda_{\text{regret}}  \qquad \text{(Since $\E_{d_{\pi_n}}  \E_{ \pi_n}  [ Q_{\pi_n}]  = (1-\gamma) J(\pi_n)$)}\\
&=    \min_{\pi \in \Pi} \frac{1}{1-\gamma} \sum_{n=1}^{N} \E_{d_{\pi_n}}  \E_{ \pi} [ (1-\lambda) Q_{\pi_n} + \lambda Q_{\pi^*}  ] - \frac{1}{1-\gamma} \sum_{n=1}^{N} \E_{d_{\pi_n}}\left[ (1-\lambda) V^*_{\pi_n} + \lambda V_{\pi^*} \right] + \frac{N}{1-\gamma} \epsilon^\lambda_{\text{regret}}  \\
&\leq \frac{N}{1-\gamma} \epsilon^{\lambda}_{\text{class}} +  \frac{N}{1-\gamma} \epsilon^\lambda_{\text{regret}}
\end{align*}
This concludes the proof.

\end{document}